\documentclass{article}

\usepackage{microtype}
\usepackage{graphicx}
\usepackage{subcaption}
\usepackage{booktabs} 

\usepackage{hyperref}
\usepackage[shortlabels]{enumitem}

\usepackage[preprint]{icml2026}

\usepackage{amsmath}
\usepackage{amssymb}
\usepackage{mathtools}
\usepackage{amsthm}

\icmltitlerunning{Efficient Soft Actor-Critic with LLM-Based Action-Level Guidance for Continuous Control}

\usepackage[capitalize,noabbrev]{cleveref}
\newtheorem{thm}{Theorem}
\newtheorem{lem}{Lemma}
\newtheorem{prop}{Proposition}

\usepackage{graphicx}
\usepackage{url}
\usepackage{subcaption}  
\usepackage{amsmath}

\usepackage[textsize=tiny]{todonotes}

\begin{document}

\twocolumn[
  \icmltitle{Efficient Soft Actor-Critic with LLM-Based Action-Level Guidance for Continuous Control}


  \icmlsetsymbol{equal}{*}


  \begin{icmlauthorlist}
    \icmlauthor{Hao Ma}{ucas,casia}
    \icmlauthor{Zhiqiang Pu}{ucas,casia}
    \icmlauthor{Xiaolin Ai}{ucas,casia}
    \icmlauthor{Huimu Wang}{jd}
  \end{icmlauthorlist}

  \icmlaffiliation{ucas}{School of Artificial Intelligence, University of Chinese Academy of Sciences, Beijing, China}
  \icmlaffiliation{casia}{Institute of Automation, Chinese Academy of Sciences, Beijing 100190, China}
  \icmlaffiliation{jd}{JD.com, Beijing, China}

  \icmlcorrespondingauthor{Zhiqiang Pu}{zhiqiang.pu@ia.ac.cn}

  \icmlkeywords{Machine Learning, Reinforcement Learning, LLM}

  \vskip 0.3in
]



\printAffiliationsAndNotice{}  

\begin{abstract}

We present GuidedSAC, a novel reinforcement learning (RL) algorithm that facilitates efficient exploration in vast state-action spaces. GuidedSAC leverages large language models (LLMs) as intelligent supervisors that provide action-level guidance for the Soft Actor-Critic (SAC) algorithm. The LLM-based supervisor analyzes the most recent trajectory using state information and visual replays, offering action-level interventions that enable targeted exploration. Furthermore, we provide a theoretical analysis of GuidedSAC, proving that it preserves the convergence guarantees of SAC while improving convergence speed. Through experiments in both discrete and continuous control environments, including toy text tasks and complex MuJoCo benchmarks, we demonstrate that GuidedSAC consistently outperforms standard SAC and state-of-the-art exploration-enhanced variants (e.g., RND, ICM, and E3B) in terms of sample efficiency and final performance.
\end{abstract}

\section{Introduction}

While reinforcement learning (RL) has demonstrated successes in robot control \citep{johannink2019residual, galljamov2022improving, zhang2024wococo, ma2023eureka}, a fundamental challenge hindering its broader application, particularly in complex robotic tasks, is its well-known high sample complexity, making exploration very inefficient. This issue is significantly amplified with vast state-action spaces, where acquiring a practically useful policy through RL entirely from scratch often proves to be an infeasible undertaking.

Exploration based on intrinsic rewards, such as count-based \citep{bellemare2016unifying, tang2017exploration}, memory-based \citep{savinov2018episodic, jiangepisodic2025}, and prediction-based \citep{burda2018exploration, pathak2017curiosity} methods, has been extensively studied to address the exploration problem by providing bonuses that encourage agents to discover novel states. However, a critical issue remains overlooked: a novel state does not necessarily equate to a valuable state, and it is the valuable states that are crucial for effective exploration. Novel states are relatively easy to encounter, but the likelihood of discovering high-value states or trajectories is exceedingly low. This renders exploration within vast state-action spaces inefficient, even with intrinsic reward methods.

To efficiently explore high-value states, another method is the imitation-based approach, which initially leverages demonstrations to learn a policy and then refines it through online RL \citep{peng2018deepmimic, galljamov2022improving}. This approach effectively learns the prior distribution of the policy and explores around this distribution. Recent studies have shown that such methods can enable humanoid robots, operating in vast state-action spaces, to learn to walk naturally. However, the collection of demonstrations is costly, and dealing with heterogeneous demonstration data is challenging. These limitations impose a barrier to the wide application of the imitation-based approach.


To address the limitations of existing exploration methods, we propose leveraging the concept of a real-time supervisor that can guide the RL learning process based on its observations. This paradigm mirrors how humans learn under supervision, where guidance, even if imperfect, can significantly accelerate learning. To realize such a supervisor capable of providing targeted guidance across diverse scenarios, we turn to large language models (LLMs). Benefiting from vast pretraining data, LLMs acquire a broad understanding of policies and principles applicable to a wide range of tasks, as demonstrated by recent works \citep{jin2024robotgpt, 10529514, yanefficient}. With targeted guidance from LLMs, RL agents can efficiently explore high-value states without relying on costly manual data collection.

We propose a theoretical framework that proves the convergence and improved sample efficiency of the SAC algorithm under action-level guidance from a suboptimal policy, upon which we design an LLM-based supervisor that adaptively generates such guidance, where one LLM provides high-level analysis and another generates low-level implementations in the form of rule-based policies. We integrate this design into SAC to develop GuidedSAC. It is worth noting that our theoretical insights are broadly applicable to value-based algorithms. However, in this paper, we focus on implementing our method based on SAC, given its superior performance in continuous control tasks compared to value-based methods. 

We evaluate GuidedSAC across both discrete toy text tasks and high-dimensional MuJoCo benchmarks. The results show that it consistently achieves better sample efficiency and final performance than standard SAC and state-of-the-art exploration methods (e.g., RND, ICM, and E3B). In summary, our contributions are twofold. (1) We present a theoretical analysis that addresses the key question: How can LLM-generated policies improve the efficiency of RL? (2) We propose the GuidedSAC algorithm, which is empirically validated across both discrete and complex continuous control tasks.



\begin{figure*}[t]
    \centering
    \includegraphics[width=0.7\linewidth]{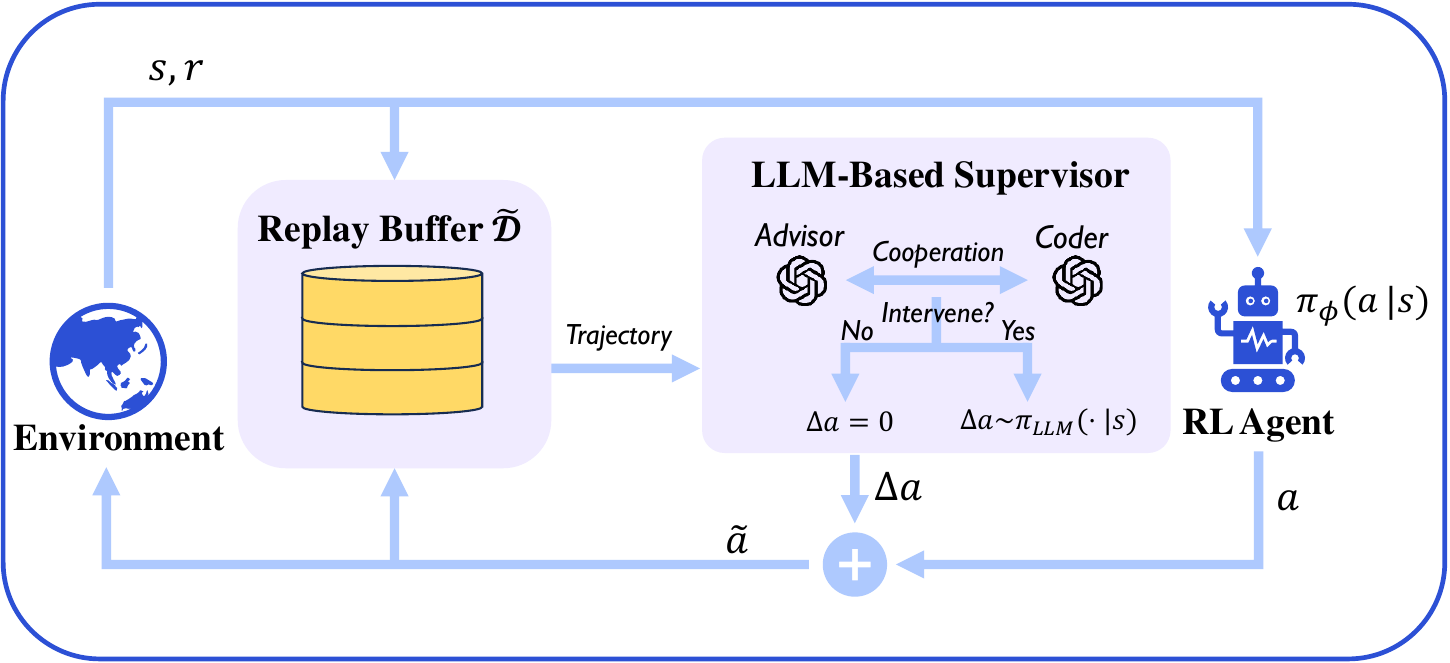}
    \caption{\textbf{The framework of GuidedSAC.} GuidedSAC leverages an LLM-based supervisor to analyze the last trajectory and determine whether intervention is necessary. If intervention is triggered, a residual action $\Delta a$ is added to the original action $a$, resulting in the intervened action $\widetilde{a}$. This intervened action is then stored in the replay buffer, facilitating the discovery of high-value trajectories.}
    \label{fig:framework}
\end{figure*}

\section{Preliminary}


\subsection{Markov Decision Process}
The Markov Decision Process (MDP) provides a fundamental framework for sequential decision-making under uncertainty, defined by the tuple $<S,A,P,R,\gamma>$. Here, $S$ is the state space, $A$ is the action space, $P:S\times A\times S\rightarrow [0,1]$ is the transition probability, $R:S\times A\rightarrow \mathbb{R}$ is the reward function, and $\gamma \in [0,1]$ is the discount factor. The agent's goal is to learn a policy $\pi:S\rightarrow A$ that maximizes the expected return:
\begin{equation}
    J(\pi) = \mathbb{E}_{\tau\sim\pi}[R(\tau)],
\end{equation}
where $\tau = (s_0,a_0,r_0,s_1,a_1,r_1,\ldots)$ is a trajectory generated by $\pi$.

\subsection{Maximum Entropy RL and SAC}
Maximum Entropy RL (MaxEnt RL) extends the standard RL objective by maximizing the policy's entropy in addition to the expected reward. This encourages exploration and can lead to better policies. The objective function is:
\begin{equation}
    J(\pi) = \mathbb{E}_{s_t\sim\rho^\pi, a_t\sim\pi}\left[\sum_{t=0}^\infty \gamma^t (r(s_t,a_t) + \alpha\mathcal{H}(\pi(\cdot|s_t)))\right],
\end{equation}
where $\rho^\pi$ is the state distribution under $\pi$, $\alpha$ is the temperature parameter, and $\mathcal{H}(\pi(\cdot|s_t)) = -\mathbb{E}_{a\sim\pi(\cdot|s_t)}[\log \pi(a|s_t)]$ is the policy entropy.

SAC is a model-free, off-policy algorithm based on MaxEnt RL. It learns a stochastic policy and a soft value function by simultaneously maximizing reward and policy entropy. SAC is known for its sample efficiency and stability in continuous action spaces, utilizing actor-critic architectures and off-policy learning from a replay buffer.

\section{Method}
Although SAC performs well on continuous control problems, it remains inefficient for many complex tasks when learning from scratch, often failing to produce the desired policy without additional guidance. LLMs, leveraging their pretraining knowledge, can provide macro-level insights by identifying areas where the current policy needs improvement and offering targeted guidance. RL can carry out fine-grained policy optimization at the micro level. From this perspective, LLM policies are complementary to RL policies.

In this section, we introduce a principled algorithm, GuidedSAC, for providing action-level guidance during RL training. We build GuidedSAC on SAC due to its strong performance in continuous control and its inherent advantage of using a Q-value-based Boltzmann policy for policy improvement. In contrast, action-level interventions in policy-gradient methods typically mask gradients at intervened steps (see Appendix~\ref{append:pg_intervene}), making them less suitable for learning directly from action-level guidance.

In GuidedSAC, we integrate SAC with an LLM-based supervisor to provide guidance as shown in Fig.~\ref{fig:framework}. The LLM-based supervisor consists of two agents: an advisor and a coder. The advisor analyzes replays of the current policy and provides advice to the coder. The coder then implements this guidance into code (rule-based policies). Through their collaboration, the LLM-based supervisor effectively provides action-level guidance, achieving improved performance in complex environments.

In the following, we first derive the theoretical motivation for GuidedSAC and then provide the implementation of GuidedSAC. Finally, we elaborate on the design of the LLM-based supervisor.

\subsection{Derivation of Guided Soft Actor-Critic}
The algorithm of GuidedSAC is as follows: When the intervention occurs at time step $t$, a state $s_t$ is passed to both $\pi_\phi(\cdot|\cdot)$ and $\pi_{\text{LLM}}(\cdot|\cdot)$. An action $a_t$ is sampled from $\pi_\phi(\cdot | s_t)$, and a residual action $\Delta a_t$ is sampled from $\pi_{\text{LLM}}(\cdot | s_t)$. The final action is computed as $a_t + \Delta a_t$. For simplicity, we denote the policy with intervention as $\pi_{\text{interv}}$, i.e. $(a_t + \Delta a_t) \sim \pi_{\text{interv}}$.

When $\pi_\phi(\cdot | s_t)$ is already performing well enough, LLM intervention may not be necessary. To determine whether intervention is needed, the advisor evaluates the performance over a recent time window and makes a decision for a period of time: $I(s_{(\lfloor t/M \rfloor - 1)M}, \ldots, s_{\lfloor t/M \rfloor M}) \in \{0, 1\}$, where $M$ denotes the trajectory length. When the advisor decides to intervene, $\pi_{\text{interv}}$ will continue intervening for the remainder of the episode. For simplicity, we denote $I(s_{(\lfloor t/M \rfloor - 1)M}, \ldots, s_{\lfloor t/M \rfloor M}) \in \{0, 1\}$ as $I(s_{\leq t})$.

To generalize this process, the intervention decision for a given state $s$ is represented as $g(s) \in \{0, 1\}$. The resulting action $\widetilde{a}_t$ can then be interpreted as being sampled from a mixed behavior policy, expressed as:
\begin{equation}
\widetilde{\pi}_\phi(\cdot | s) = g(s) \pi_{\text{interv}}(\cdot | s) + (1 - g(s)) \pi_\phi(\cdot | s).
\end{equation}

The advisor LLM analyzes replays from the recent past (including both state and images) to decide whether intervention is necessary and the coder LLM provides guidance on which aspects require adjustment. The update equations are given through the following lemmas. We refer readers to Appendix~\ref{appendix:proof} for full proofs.

\begin{lem}[Guided Policy Evaluation]
    Consider the guided Bellman backup operator $\mathcal{T}^{\widetilde{\pi}}$ and a mapping $Q:\mathcal{S}\times\mathcal{A}\rightarrow\mathbb{R}$, 
    and define $Q^{k+1}=\mathcal{T}^{\widetilde{\pi}} Q^k$. Then the sequence $Q^k$ will converge to the Q-value of $\widetilde{\pi}$ as $k\rightarrow\infty$.
\end{lem}


\begin{lem}[Guided Policy Improvement]
    Let $\widetilde{\pi}_{\mathrm{old}} \in \Pi$ and update policy following Equation~(\ref{eq:policy_update}) to get $\pi_{\mathrm{new}}$. 
    Then $V^{\pi_{\mathrm{new}}}\left(s_t\right) \geq V^{\widetilde{\pi}_{\mathrm{old}}}\left(s_t\right)$ for all $s_t\in \mathcal{S}$.
\end{lem}

\begin{equation} \label{eq:policy_update_app}
    \begin{aligned}
\pi_{\mathrm{new}}\left(\cdot\mid s_t\right) & =\arg \min _{\pi^{\prime} \in \Pi} D_{\mathrm{KL}}\left(\pi^{\prime}\left(\cdot \mid s_t\right) \Bigg\Vert 
\frac{\exp Q^{\widetilde{\pi}_{\mathrm{old}}}\left(s_t, \cdot\right)}{\log Z^{\widetilde{\pi}_{\mathrm{old}}}\left(s_t\right)}\right) \\
& =\arg \min _{\pi^{\prime} \in \Pi} J_{\widetilde{\pi}_{\mathrm{old}}}\left(\pi^{\prime}\left(\cdot \mid s_t\right)\right),\ \forall s_t \in S .
\end{aligned}
\end{equation}

From the above two lemmas, we can tell that the guidance in the behavior policy $\widetilde{\pi}$ does not affect the convergence of $\pi$. This key finding directly informs our design: we can freely guide the behavior policy using any intervention policy $\pi_{\text{interv}}$. By training on the data it collects, $\pi$ is guaranteed to converge, which can be expressed as the following theorem.

\begin{thm}[Convergence of GuidedSAC]
\label{theorem:conv}
    By repeatedly applying the guided policy evaluation in Equation~(\ref{eq:policy_evaluation}) and guided policy improvement in Equation~(\ref{eq:policy_update}), 
    the policy network $\pi_{\mathrm{new}}$ converges to the optimal policy $\pi^*$ if $V^{\widetilde{\pi}}\left(s\right) \geq V^{\pi}\left(s\right)$ for all $s \in S$.
\end{thm}
The assumption that $V^{\widetilde{\pi}}\left(s\right) \geq V^{\pi}\left(s\right)$ for all $s \in S$ is reasonable, as we can always choose not to intervene when $\pi$ is sufficiently good, thereby ensuring the assumption remains satisfied. In the following proposition, we derive a theoretical analysis to show how the quality of the guidance policy $\pi_{\text{interv}}$ and intervention decision $g(\cdot)$ affect the efficiency of the GuidedSAC algorithm.
\begin{prop}[Single Step Improvement] \label{prop:single_step}
    Under the assumption that the guidance is superior to the current policy, that is, $V^{\pi_{\mathrm{interv}}}(s) \geq V^{\pi_{\mathrm{old}}}(s)$ if $g(s) > 0$, 
    the lower bound of the improvement is higher than without guidance.
\end{prop}

\begin{proof}
    By definition, $\widetilde{\pi}_{old}(\cdot|s)=g(s)\pi_{\text{interv}}(\cdot|s) + (1-g(s))\pi_{\mathrm{old}}(\cdot|s)$. 
    We can decompose $V^{\widetilde{\pi}_{\mathrm{old}}}$ in Equation~(\ref{eq:Bellman improvement}) such that
    \begin{equation}
        \begin{aligned}
            & Q^{\widetilde{\pi}_{\mathrm{old}}}\left(s_t, a_t\right) =r\left(s_t, a_t\right)+\gamma \mathbb{E}_{s_{t+1} \sim p}\left[V^{\widetilde{\pi}_{\mathrm{old}}}\left(s_{t+1}\right)\right] \\
            & =r\left(s_t, a_t\right)+\gamma \mathbb{E}_{s_{t+1} \sim p} \Big[g(s_{t+1})V^{\pi_{\text{interv}}}\left(s_{t+1}\right) \\
            & \quad + \big(1-g(s_{t+1})\big) V^{\pi_{\mathrm{old}}}\left(s_{t+1}\right) \Big] \\
            & \leq Q^{\pi_{\mathrm{new}}}\left(s_t, a_t\right).
        \end{aligned}
    \end{equation}
\end{proof}

When $g(s_{t+1}) = 0$, GuidedSAC is equivalent to SAC. When $g(s_{t+1}) = 1$, the improvement is guaranteed if $V^{\pi_{\text{interv}}}(s_{t+1}) \geq V^{\pi_{\mathrm{old}}}(s_{t+1})$.
That is, $\pi_{\text{interv}}$ does not need to be optimal to guide $\pi_{\mathrm{old}}$, it only needs to ensure that $\pi_{\text{interv}}$ is better than $\pi_{\mathrm{old}}$ when guidance occurs to achieve reliable improvement. 
This insight suggests that two elements are crucial for GuidedSAC: a good intervention policy $\pi_{\text{interv}}$ and a good $g(\cdot)$ to identify when to guide.

To summarize, Lemmas~1 and 2, along with Theorem~1, demonstrate that SAC can still converge under action-level interventions, as long as the mixed policy \( \widetilde{\pi} \) outperforms the original policy \( \pi \). The function \( g(\cdot) \) ensures this condition is met, even in difficult situations. Proposition~1 reveals that intervention policies need not be optimal to be beneficial: as long as the intervention policy \( \pi_{\text{interv}} \) outperforms the current policy locally, it can significantly enhance sample efficiency. This highlights the strength of the proposed framework, which uses imperfect but helpful guidance to speed up learning and tackle challenges in RL for complex continuous control tasks.

\subsection{Guided Soft Actor-Critic}
Based on the previous derivations, we implement GuidedSAC Equation~(\ref{eq:policy_evaluation}) and Equation~(\ref{eq:policy_update}) with neural networks. 
In the part of state value estimation, we use TD loss to update the state value network and the Q network. Since the algorithm is based on SAC, the state value will have an item related to the entropy of the policy. The update of soft value estimation is as follows:

\begin{equation}
\begin{split}
     J_V(\psi) = \mathbb{E}_{\mathbf{s}_t \sim \widetilde{\mathcal{D}}} \bigg[ \frac{1}{2} \Big( \scriptstyle V_\psi(\mathbf{s}_t) - \mathbb{E}_{\mathbf{a}_t \sim \pi_\phi} \big[ Q_\theta(\mathbf{s}_t, \mathbf{a}_t) \\
     - \log \pi_\phi(\mathbf{a}_t \mid \mathbf{s}_t) \big] \Big)^2 \bigg],
\end{split}
\end{equation}
where $\widetilde{\mathcal{D}}$ is the replay buffer that mixed with guided trajectories. The state value network is not necessary, but it can improve the stability of action value estimation. The update of critic network follows TD loss

\begin{equation}
    J_Q(\theta)=\mathbb{E}_{\left(\mathbf{s}_t, \mathbf{a}_t\right) \sim \widetilde{\mathcal{D}}}\left[\frac{1}{2}\left(Q_\theta\left(\mathbf{s}_t, \mathbf{a}_t\right)-\hat{Q}\left(\mathbf{s}_t, \mathbf{a}_t\right)\right)^2\right],
\end{equation}
\begin{equation}
    \hat{Q}\left(\mathbf{s}_t, \mathbf{a}_t\right)=r\left(\mathbf{s}_t, \mathbf{a}_t\right)+\gamma \mathbb{E}_{\mathbf{s}_{t+1} \sim p}\left[V_{\bar{\psi}}\left(\mathbf{s}_{t+1}\right)\right],
\end{equation}
where the $\bar{\psi}$ is a target value network and defined by a soft update $\bar{\psi} = \tau\psi + (1-\tau)\bar{\psi}$.
By expanding Equation~(\ref{eq:policy_update}), the loss function of actor network is
\begin{equation}
    J_\pi(\phi)=\mathbb{E}_{\mathbf{s}_t \sim \widetilde{\mathcal{D}}, \epsilon_t \sim \mathcal{N}}\left[\scriptstyle \log \pi_\phi\left(f_\phi\left(\epsilon_t ; \mathbf{s}_t\right) \mid \mathbf{s}_t\right)-Q_\theta\left(\mathbf{s}_t, f_\phi\left(\epsilon_t ; \mathbf{s}_t\right)\right)\right].
\end{equation}

Compared to SAC, the loss function of GuidedSAC differs in that the replay buffer $\widetilde{\mathcal{D}}$ contains samples drawn from a mixed policy $ \widetilde{\pi}_\phi$. According to Theorem~\ref{theorem:conv} and Proposition~\ref{prop:single_step}, sampling from $ \widetilde{\pi}_\phi$ does not affect convergence, moreover, if $\pi_{\text{interv}}$ is superior to $\pi_\phi$, it can even accelerate convergence. The pseudocode of GuidedSAC is provided in Alg.~\ref{alg:soft_actor_critic}.

\begin{figure}[t]
    \centering
    \includegraphics[width=0.9\linewidth]{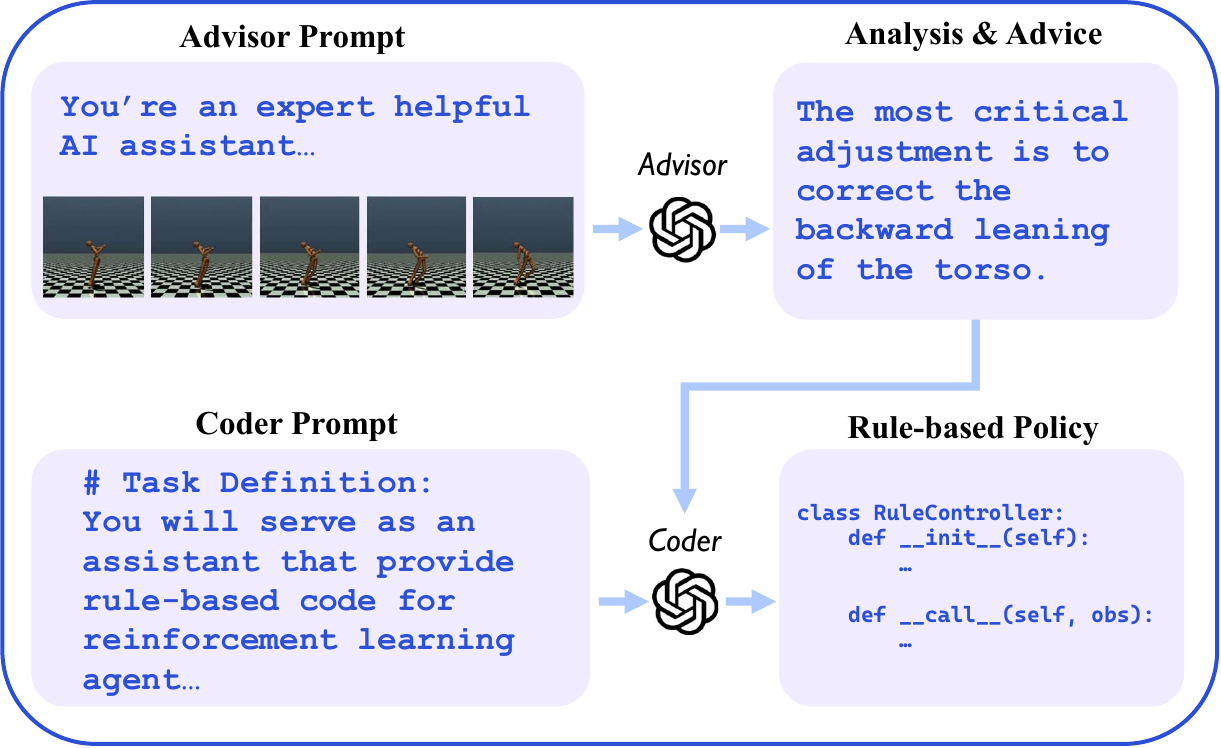}
    \caption{\textbf{Illustration of the LLM-based supervisor's cooperation details.}}
    \label{fig:llm-based-supervisor}
\end{figure}

\subsection{LLM-based Supervisor}


Recent research has shown that task decomposition, where multiple LLMs cooperate to complete a task, is more effective than relying on a single LLM to handle the entire task \citep{guo2024large}. Inspired by this, the LLM-based supervisor is designed to include two LLMs: an \textit{advisor} observes the policy replay and provides suggestions, and a \textit{coder} generates a rule-based policy based on those suggestions, as illustrated in Fig.~\ref{fig:llm-based-supervisor}. When the advisor LLM finds that the RL policy is good enough, it decides not to intervene.

GuidedSAC relies on well designed prompts to define $\pi_{\text{interv}}$ and $g(\cdot)$ effectively. We use a five component framework consisting of a task definition, background information, chain of thought reasoning, domain hints, and a code template. The task definition establishes the agent role and objectives while the background information provides environmental context and documentation. Chain of thought reasoning supports systematic analysis and domain hints improve accuracy by addressing specific LLM limitations. Finally, the code template ensures the output follows a structured format. The Advisor utilizes the first four elements whereas the Coder includes the code template to maintain implementation consistency. Contemporary LLM agent architectures often incorporate these five elements into their prompt designs \citep{wang2025mobile}. Detailed prompts are provided in Appendix \ref{sec:prompt_details}.

\section{Experiment}

We design our experiments to systematically validate GuidedSAC through a progression of increasing complexity. First, we evaluate in \textit{toy text environments} with discrete action spaces using direct policy substitution $\widetilde{\pi} = \pi_{\text{LLM}}$ rather than residual action. This simplification allows us to isolate and validate the core idea that LLM-based guidance can accelerate learning. Among baseline exploration methods, RND demonstrates the strongest overall performance on these discrete tasks. Building on this foundation, we then evaluate GuidedSAC on continuous control benchmarks \textit{MountainCar} and \textit{Humanoid} using the full intervention mechanism with residual actions, where RND serves as the primary exploration baseline due to its success in the first experiment. Finally, ablation studies analyze intervention timing and duration to understand optimal guidance policies.

The progression from discrete to continuous action spaces is designed to validate that LLM guidance provides both sample efficiency gains and the ability to generate interpretable, task-appropriate policies. By establishing strong performance in toy text environments where guidance effects can be cleanly isolated, we build confidence that the approach scales to complex, real-world relevant control problems.

\subsection{Discrete Control Tasks}
\textbf{Setup}. To evaluate the effectiveness of GuidedSAC compare it with state-of-the-art exploration methods, we first conduct experiments on four classic discrete control tasks from toy text environments.

The toy text environments include \textit{Blackjack}, a card game where the agent must learn optimal decision making under uncertainty. \textit{CliffWalking} represents a grid world navigation task requiring the agent to find the shortest path while avoiding a cliff. \textit{FrozenLake} provides a stochastic grid world environment where the agent must navigate across a frozen lake with holes. \textit{Taxi} presents a domain where the agent must pick up and deliver passengers while navigating a grid world. These environments are characterized by discrete state and action spaces, sparse rewards, making them ideal testbeds for comparing exploration strategies.

In these tasks, we employ direct policy substitution $\widetilde{\pi} = \pi_{\text{LLM}}$ instead of residual action intervention. This simplification allows us to isolate the effect of LLM-based guidance from the intervention mechanism. By doing so, we can directly validate the core hypothesis that semantic, task-aware guidance accelerates learning more effectively than undirected, novelty-based exploration.

\textbf{Baselines}. We compare GuidedSAC against three prominent intrinsic reward based exploration methods. \textit{RND (Random Network Distillation)} \citep{burda2018exploration} encourages exploration by rewarding the agent for visiting states that are novel to a randomly initialized neural network. \textit{E3B (Exploration via Elliptical Episodic Bonuses)} \citep{henaff2022exploration} uses elliptical bonuses to guide exploration in episodic settings. \textit{ICM (Intrinsic Curiosity Module)} \citep{pathak2017curiosity} leverages prediction errors as intrinsic rewards to drive exploration. These methods represent the current state of the art in exploration driven reinforcement learning and provide a comprehensive baseline for evaluating the effectiveness of LLM-based guidance.

\begin{figure}[t]
    \centering
    \begin{subfigure}{0.48\linewidth}
        \centering
        \includegraphics[width=\linewidth]{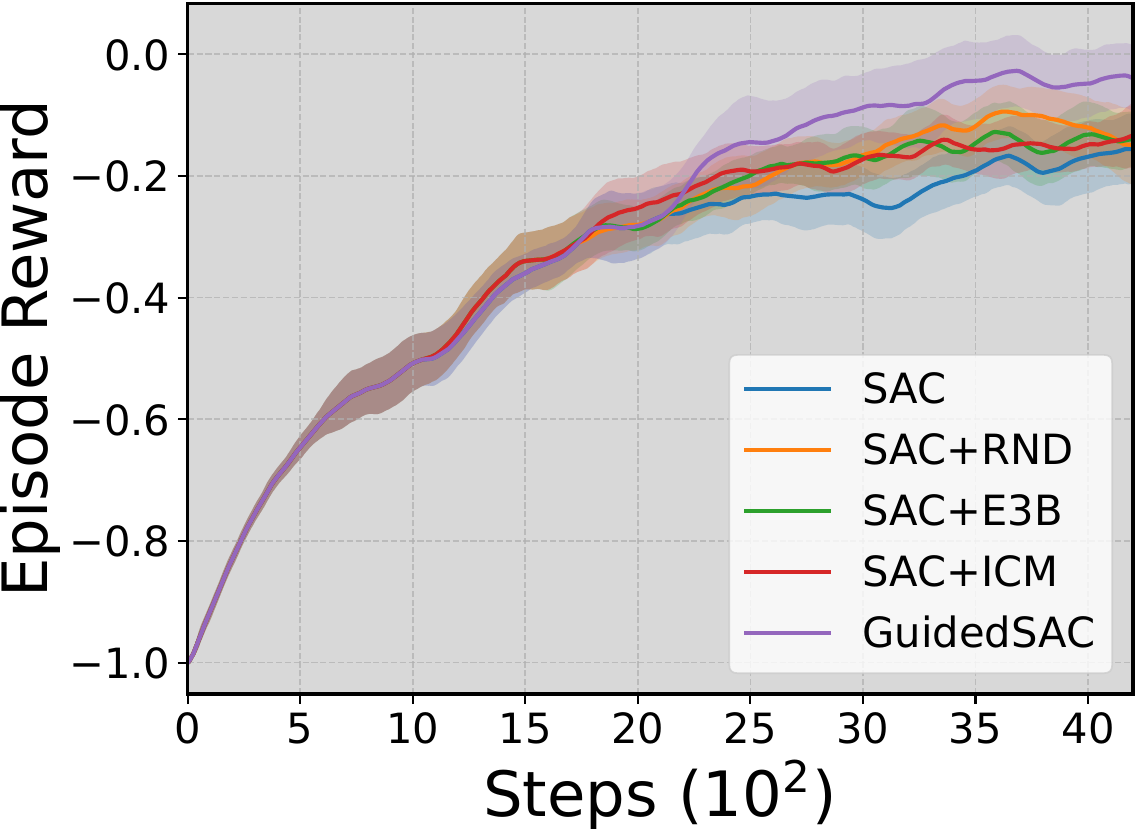}
        \caption{Blackjack}
        \label{fig:blackjack_results}
    \end{subfigure}
    \hfill
    \begin{subfigure}{0.48\linewidth}
        \centering
        \includegraphics[width=\linewidth]{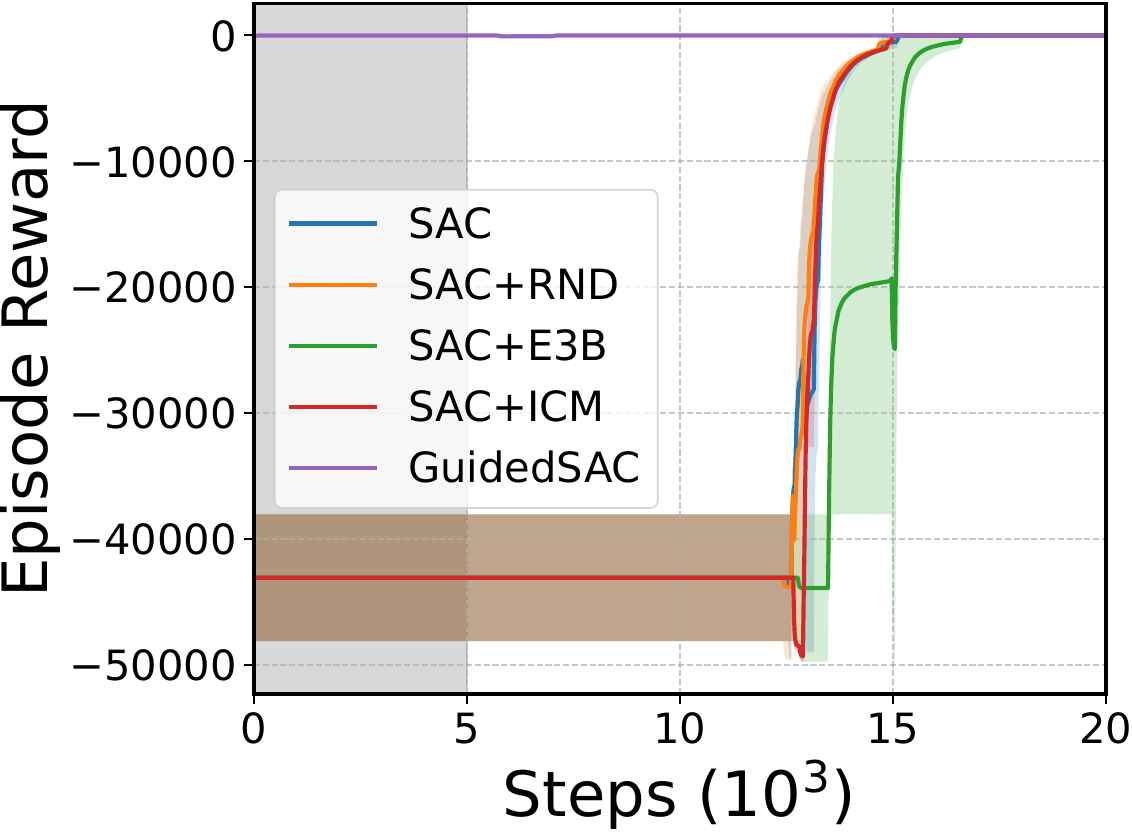}
        \caption{CliffWalking}
        \label{fig:cliffwalking_results}
    \end{subfigure}
    \vspace{0.5cm}
    \begin{subfigure}{0.48\linewidth}
        \centering
        \includegraphics[width=\linewidth]{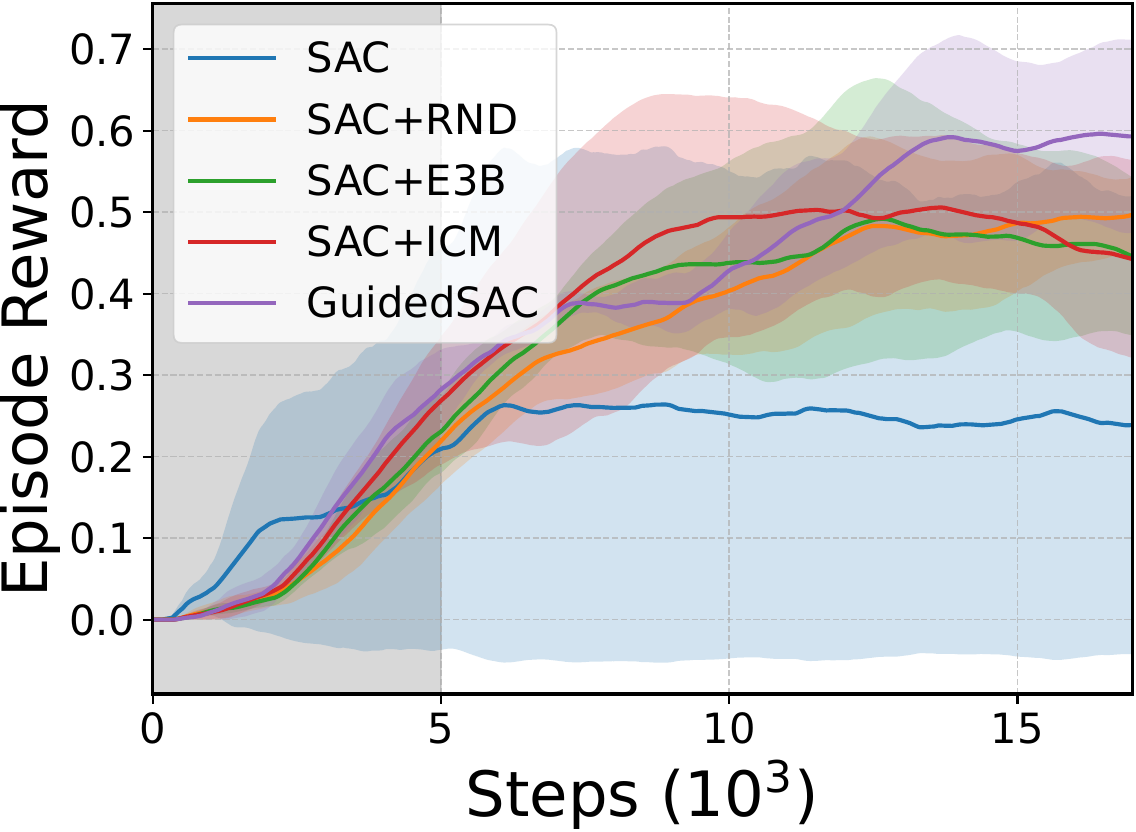}
        \caption{FrozenLake}
        \label{fig:frozenlake_results}
    \end{subfigure}
    \hfill
    \begin{subfigure}{0.48\linewidth}
        \centering
        \includegraphics[width=\linewidth]{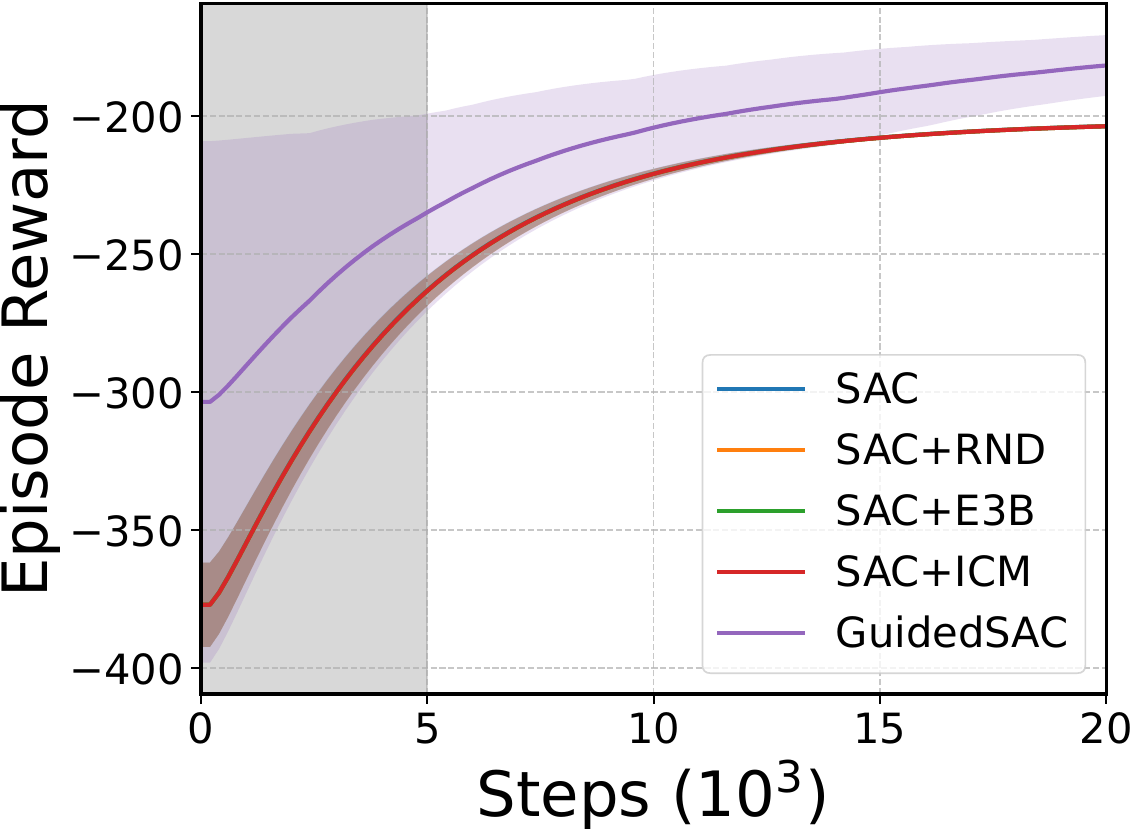}
        \caption{Taxi}
        \label{fig:taxi_results}
    \end{subfigure}
    \caption{\textbf{Performance comparison on toy text.} Training curves comparing GuidedSAC with RND, E3B, and ICM across four toy text environments. The shaded regions represent the predefined intervals where intervention can occur.}
    \label{fig:toy_text_results}
\end{figure}

\textbf{Efficiency and Convergence Analysis.} In Fig.~\ref{fig:toy_text_results}, the shaded regions represent the predefined intervals where intervention can occur. Within these intervals, the timing of intervention is autonomously determined by the Advisor. As shown, GuidedSAC achieves the highest reward across all four tasks. Notably, intrinsic reward methods do not necessarily accelerate convergence in all tasks. For instance, in the CliffWalking task, the E3B method actually slows down convergence. This may be attributed to the task's relatively small exploration space and a single optimal trajectory, where random sampling before policy updates is sufficient to discover this trajectory, making intrinsic rewards introduce unnecessary exploration overhead. For CliffWalking and Taxi, GuidedSAC generates the optimal rule-based policy at the beginning, enabling near-immediate convergence to the optimal policy in these environments.

Among the baselines, RND emerges as the strongest competitor and demonstrates robust performance in structured environments like \textit{Blackjack}. However, the fundamental advantage of GuidedSAC lies in its value-aligned exploration. While intrinsic methods like RND and ICM drive the agent toward any novel state, GuidedSAC leverages the semantic understanding of the LLM to direct exploration toward trajectories that are both novel and task-relevant. In \textit{FrozenLake}, for example, novelty-driven exploration might lead to discovering new ways to fall into holes, whereas GuidedSAC focuses on reaching the goal. These results empirically support our theoretical analysis in Proposition~\ref{prop:single_step} by proving that even non-optimal LLM guidance significantly improves sample efficiency by outperforming the local random policy.

\subsection{Continuous Control Tasks}

\textbf{Setup}. We evaluate GuidedSAC on two distinct continuous control problems: MountainCar and Humanoid. The MountainCar environment requires an underpowered vehicle to ascend a steep hill by oscillating back and forth to generate momentum. Its state space is a two-dimensional continuous space representing position and velocity, while the agent exerts acceleration within the range of -1.0 to 1.0. The reward for MountainCar includes a hill reward  for reaching the goal and a control cost . In contrast, Humanoid is a high-dimensional task in MuJoCo, where a 17-joint robot must learn to walk forward. The 47-dimensional state space includes the robot's center of mass position, velocity, angular momentum, and joint angles. The reward function for Humanoid consists of a forward reward , a healthy reward , and a control cost . For both tasks, we employ residual action intervention  to enable fine-grained control. This allows the RL policy to retain autonomy while receiving targeted guidance from the LLM, which is particularly crucial in Humanoid to overcome the challenge of coordinating 17 continuous control variables.

\begin{figure}[tb]
    \centering
    \begin{subfigure}{0.48\linewidth}
        \centering
        \includegraphics[width=\linewidth]{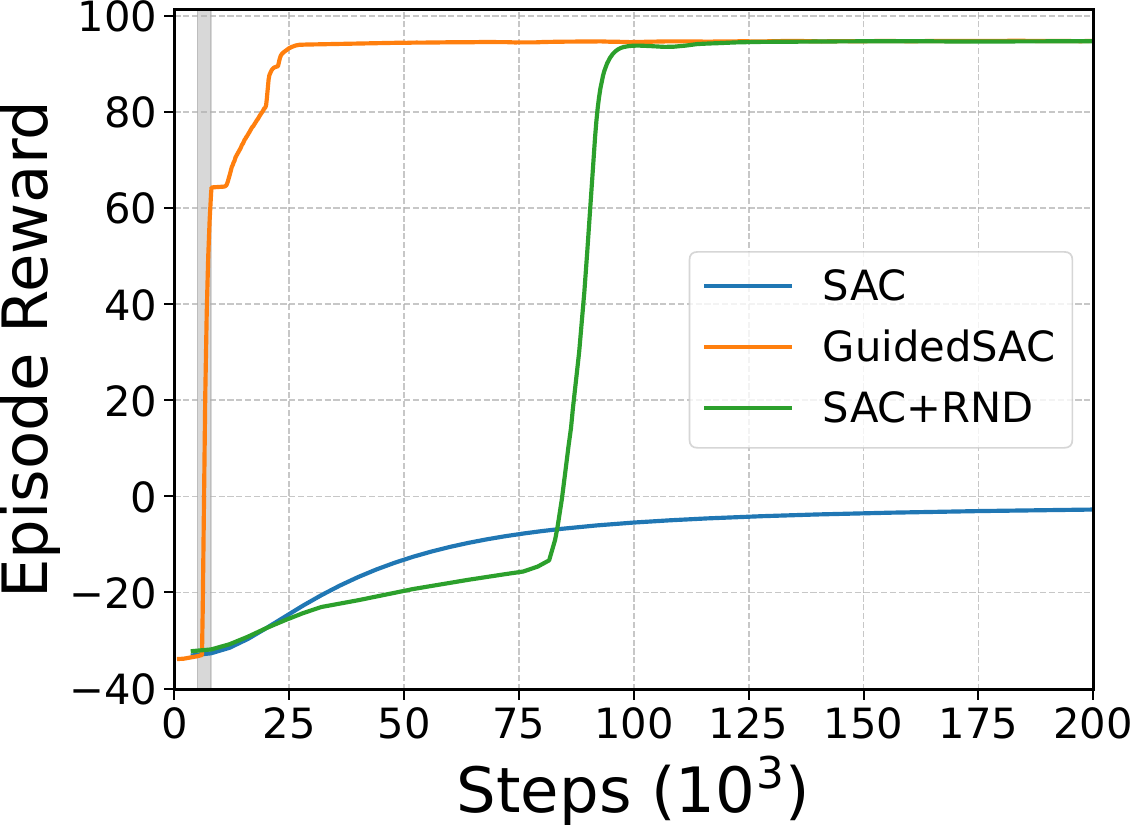}
        \caption{MountainCar}
        \label{fig:exp_mountaincar_main}
    \end{subfigure}
    \begin{subfigure}{0.48\linewidth}
        \centering
        \includegraphics[width=\linewidth]{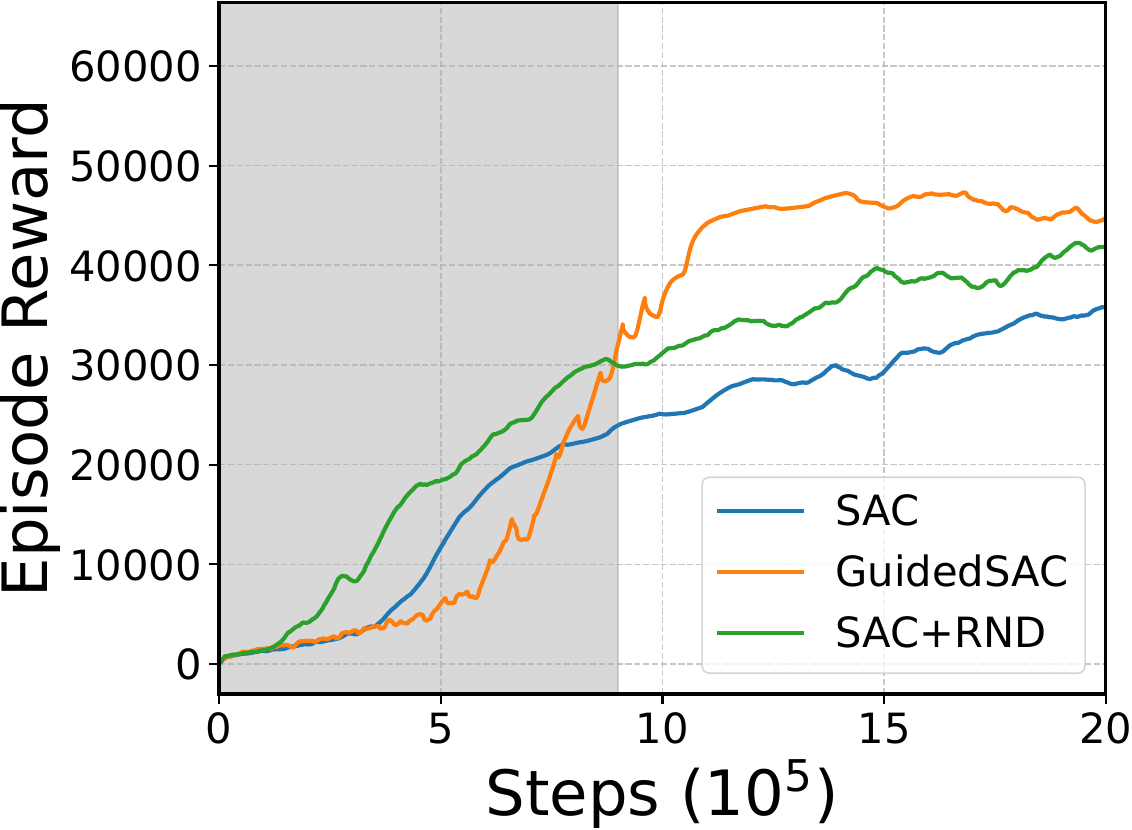}
        \caption{Humanoid}
        \label{fig:exp_humanoid_main}
    \end{subfigure}%
    \caption{\textbf{Training curves for MountainCar and Humanoid.} Shaded regions indicate intervention periods for GuidedSAC. For MountainCar the intervention occurs between steps 50k and 53k while for Humanoid it occurs between steps 700k and 800k. MountainCar reward of 100 indicates successful goal achievement. Humanoid values above 5000 represent robust bipedal locomotion.}
\end{figure}

\textbf{Baselines}. Based on the discrete control experiments where RND demonstrated the strongest performance among intrinsic reward methods, we employ both SAC and SAC+RND as baselines to evaluate the effectiveness of LLM-based guidance in continuous control settings.

\textbf{Results.} Training curves for continuous control are shown in Fig.~\ref{fig:exp_mountaincar_main} and Fig.~\ref{fig:exp_humanoid_main}. In MountainCar, SAC fails to climb the slope and remains stationary to minimize control costs. While SAC + RND eventually discovers a successful trajectory after 80k steps, GuidedSAC achieves high rewards immediately after the first intervention. This demonstrates that targeted guidance allows the agent to find high value trajectories without the need for extensive random exploration. The Humanoid task presents a greater challenge due to its high dimensional state and action spaces. During early training, the agent struggles to balance even with guidance, which results in a slower reward increase compared to SAC. However, once the policy achieves basic stability midway through training, the action level intervention becomes highly effective. GuidedSAC shows a sharp increase in reward after $7 \times 10^5$ steps and eventually surpasses all baselines. These results support our single step improvement proposition and show that LLM guidance is most effective when the baseline policy reaches a sufficient level of basic competence.

\begin{figure*}[thbp]
    \centering
    \includegraphics[width=\linewidth]{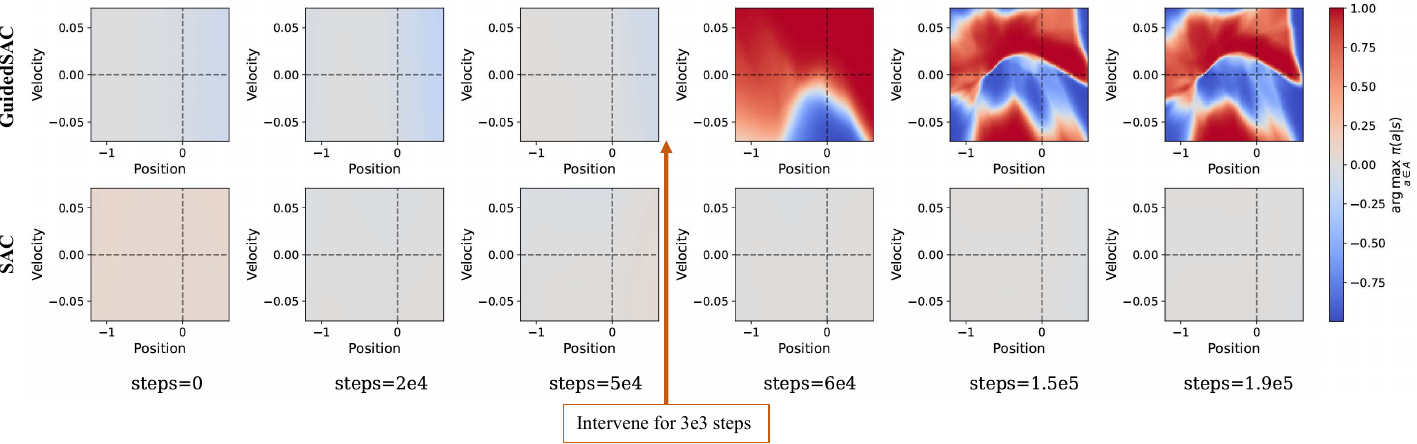}
    \caption{\textbf{Policy landscape evolution in MountainCar.} Columns show snapshots at 0k, 50k intervention start, 53k intervention end, 60k, 80k, and 100k training steps. Horizontal axis shows car position with $x \in [-1.3, 0.7]$. Vertical axis shows velocity with $v \in [-0.07, 0.07]$. Color indicates most probable action $\arg\max_a \pi(a|s)$. Dotted lines intersect at initial state with $x=-0.5$ and $v=0$. The intervention window causes immediate policy reconfiguration, demonstrating efficient knowledge transfer from $\pi_{\text{interv}}$ to $\pi_\phi$.}
    \label{fig:mountaincar_policy}
\end{figure*}

\begin{figure}[thbp]
    \centering
    \begin{subfigure}{\linewidth}
        \centering
        \includegraphics[width=\linewidth]{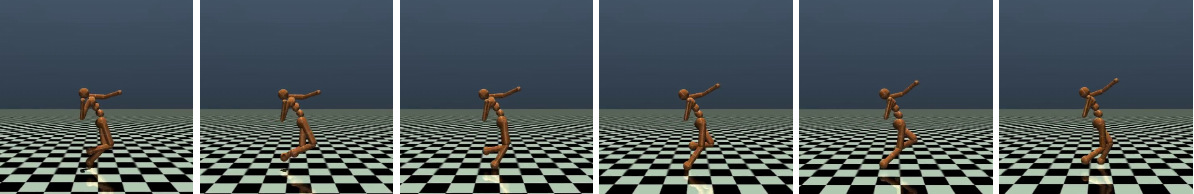}
        \caption{GuidedSAC}
    \end{subfigure}
    \hfill
    \begin{subfigure}{\linewidth}
        \centering
        \includegraphics[width=\linewidth]{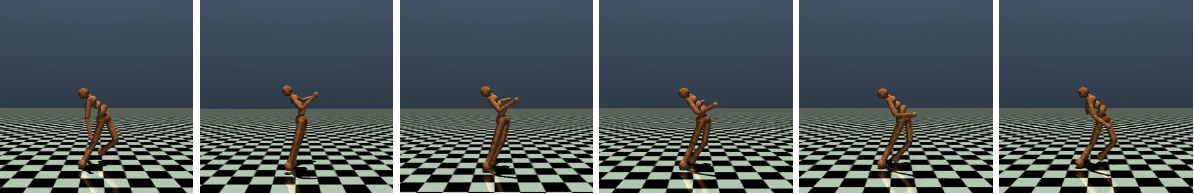}
        \caption{SAC}
    \end{subfigure}%
    \caption{\textbf{Visualization of final policies in Humanoid task.} Left: GuidedSAC produces coordinated bipedal running with sinusoidal leg motion patterns. Right: SAC converges to an unnatural gait—supporting on one leg while performing rapid small-step hops. Both achieve reward, but GuidedSAC's locomotion aligns more closely with natural human movement.}
    \label{fig:humanoid_video}
\end{figure}

\textbf{Guidance Accelerates Discovery of High-Value Trajectories} The success across both tasks highlights how LLM guidance offers targeted intervention compared to the systematic but undirected exploration of intrinsic rewards. In MountainCar, the policy landscape visualization in Fig.~\ref{fig:mountaincar_policy} shows that the policy undergoes immediate reconfiguration during the intervention window. The intervention policy $\pi_{\text{interv}}$ sets a simple oscillation strategy that is quickly retained by $\pi_\phi$ after the intervention ends. This validates Lemma~\ref{lem:guided_policy_evaluation} as data collected under the mixed policy $\widetilde{\pi}$ effectively transfers knowledge to the agent. In Humanoid, this improvement is more delayed but equally dramatic. This phenomenon reflects the interplay between task complexity and algorithmic design. During early training, the robot struggles with basic balance, meaning even with action-level intervention, the combined policy may not immediately find high-value trajectories. However, once the base policy $\pi_\phi$ achieves basic stability, the residual guidance (such as sinusoidal joint control) enhances the existing policy rather than overriding it, leading to rapid convergence.

\textbf{Qualitative Analysis of Policy Reconfiguration}. The effectiveness of GuidedSAC is further evidenced by the quality of the learned behaviors. In MountainCar, the reward curve of SAC suggests it converges to a stationary policy to minimize control costs, while GuidedSAC discovers the non-intuitive oscillation strategy. In Humanoid, as shown in Fig.~\ref{fig:humanoid_video}, the qualitative difference is even more pronounced. Standard SAC often converges to an unnatural gait where the robot supports itself on one leg while performing rapid small-step hops. Although this maximizes reward, it is biomechanically inefficient. In contrast, GuidedSAC generates more structured and interpretable locomotion. By leveraging simple rule-based guidance for sinusoidal leg motion, the agent learns coordinated bipedal running that aligns closely with natural human movement. This demonstrates that LLM-based guidance can steer the learning process toward semantically meaningful and task-appropriate policies rather than just reward-maximizing ones.

\begin{figure}[tbp]
    \centering
    \begin{subfigure}{0.47\linewidth}
        \centering
        \includegraphics[width=\linewidth]{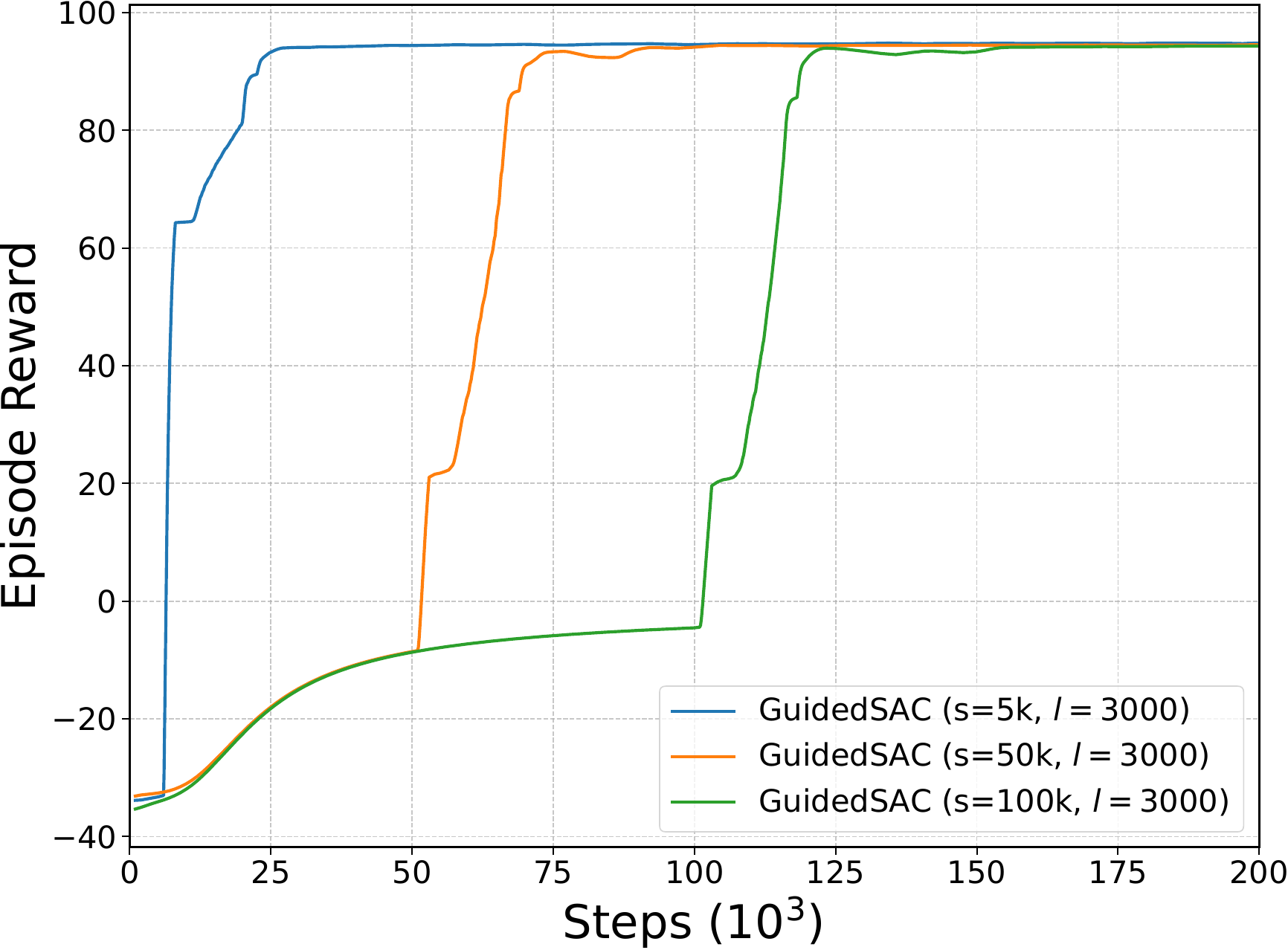}
        \caption{Ablation on intervention start point}
        \label{fig:ab_start}
    \end{subfigure}
    \hfill
    \begin{subfigure}{0.47\linewidth}
        \centering
        \includegraphics[width=\linewidth]{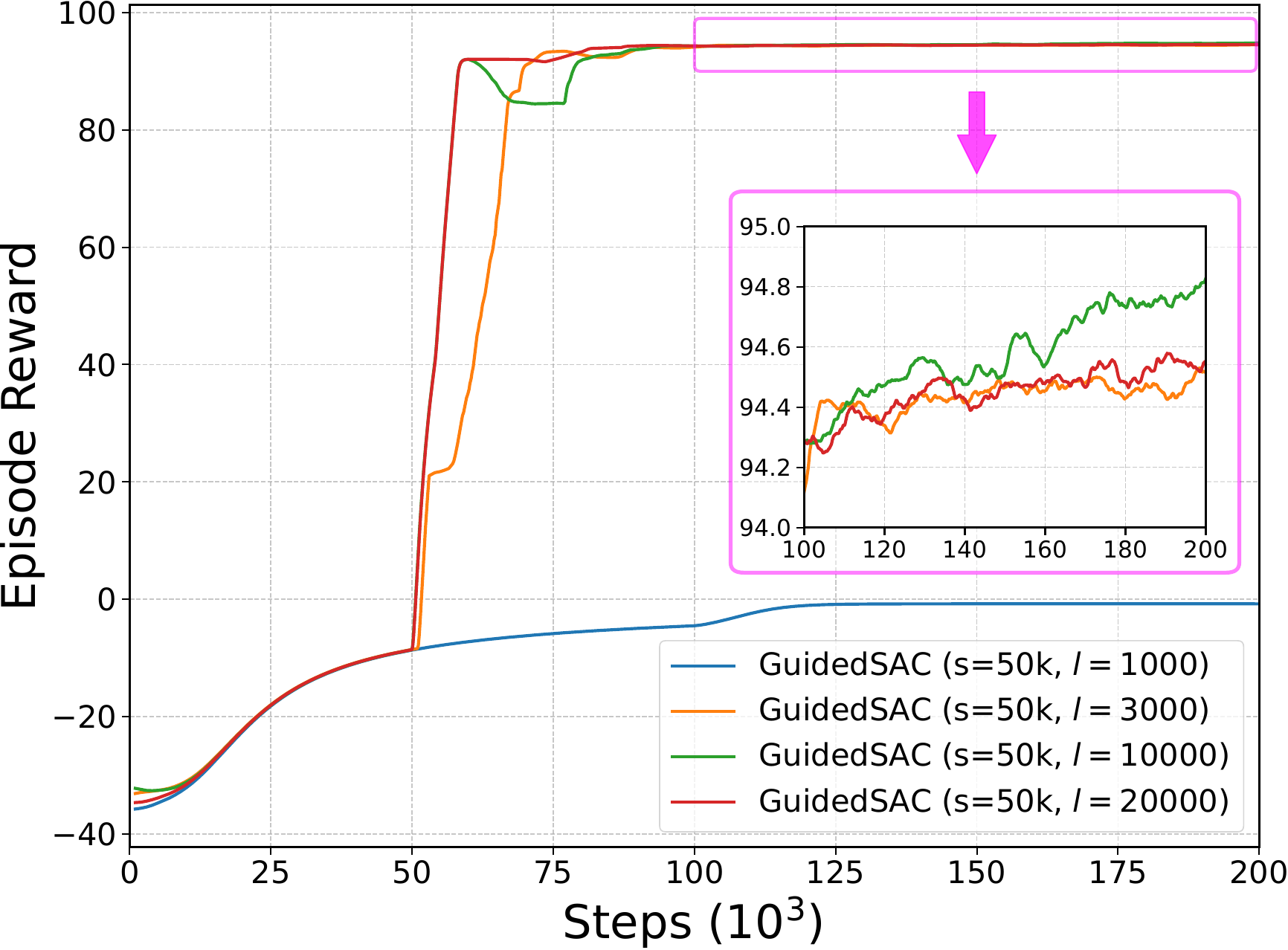}
        \caption{Ablation on intervention duration}
        \label{fig:ab_duration}
    \end{subfigure}%
    \caption{\textbf{Impact of intervention timing and duration.} Left side shows intervention start point $s$ tested at 5k, 10k, 20k, and 40k steps. Right side shows intervention duration $l$ tested at 1k, 3k, 5k, and 10k steps. Optimal performance requires early intervention with moderate duration. If duration is too short it dilutes impact, while too long hinders autonomous exploration.}
    \label{fig:ablation}
\end{figure}

\subsection{Ablation Study}

According to the single-step improvement proposition, action-level intervention in GuidedSAC should not persist throughout the entire training phase. To systematically investigate the influence of intervention timing and duration on agent performance, we conduct an ablation study using the MountainCar environment. In these experiments, we disable the autonomous judgment of the Advisor LLM and instead manually configure the intervention start point $s$ and the intervention duration $l$. These two factors collectively determine the degree of external guidance provided to the agent. For example, a configuration where $s=5000$ and $l=3000$ indicates that the intervention begins precisely at step 5000 and continues for a fixed window of 3000 steps. This manual override allows us to isolate how different temporal interventions affect the convergence and stability of the baseline policy.

\textbf{Effect of Intervention Timing}. The results of the ablation study on $s$ (Fig.~\ref{fig:ab_start}) indicate that, if the intervention policy is sufficiently effective, earlier intervention yields better results. This aligns with the single-step-improvement proposition, which suggests that timely intervention can accelerate policy improvement. 

\textbf{Effect of Intervention duration}. The ablation study on $l$ (Fig.~\ref{fig:ab_duration}) reveals that the intervention duration should be neither too short nor too long. If the duration is too short, the intervened data in $\widetilde{D}$ becomes diluted by non-intervened data, thereby diminishing its impact on policy improvement. On the other hand, if the duration is too long, after the RL policy is good enough, the intervention will hinder the RL policy from exploring better trajectories.

\section{Related Work}
\textbf{RL for Continuous Control}. Exploration in continuous control problems poses significant challenges due to the vast state-action space. Take humanoid bipedal locomotion in MuJoCo as an example. The policies produced by RL often remain suboptimal, even with advanced exploration techniques, leading to unnatural forward movement \citep{lillicrap2015ddpg, fujimoto2018td3, schulman2017ppo, haarnoja2018sac}.

To enhance policy quality, researchers have explored reference-based and reference-free paradigms leveraging external data or prior knowledge. Reference-based methods, such as DeepMimic \citep{peng2018deepmimic}, use motion capture to enable realistic locomotion, while \citep{galljamov2022improving} demonstrate that action space representation, symmetry priors, and cliprange scheduling accelerate training and improve human-like walking. Reference-free methods, relying on carefully designed reward functions, capture effective bipedal movement characteristics but often require extensive tuning or human demonstrations \citep{siekmann2021sim, van2024revisiting}.

Combining exploration techniques with prior knowledge may offer an efficient alternative. GuidedSAC introduces an LLM-based supervisor that provides action-level guidance, leveraging LLMs' prior knowledge without extensive reward tuning or human demonstrations. Its adaptability through prompt and template modifications enables integration of domain-specific knowledge, making it a promising approach for complex environments.

\textbf{Advice-taking Agents}. Advice-taking RL agents enhance learning through external guidance, such as demonstrations or language instructions. For demonstration-based advice-taking, \citep{maclin1994advice} introduced a framework that utilizes pre-defined code to guide the learning process. Wu et al.
 \citep{wu2023human} proposed a human-in-the-loop framework for autonomous navigation, allowing human operators to intervene during training. Similarly, \citep{peng2024learning} developed a method to learn a proxy value function from human interventions, which is subsequently used to guide Q-network updates. \citep{cederborg2015policy} extended Q-learning by transforming human demonstrations into weights within a Boltzmann policy. Wang et al. \citep{wang2018intervention} incorporated an imitation loss into PPO to enhance policy updates. For language-based advice-taking, recent advancements have demonstrated the potential of LLMs in shaping RL exploration. Du et al. \citep{du2023guiding} introduced ELLM, a method that rewards agents for achieving sub-goals suggested by LLM, leading to improved performance in various tasks. Ma et al. \citep{ma2024explorllm} proposed ExploRLLM, which integrates LLMs for candidate selection in manipulation tasks, enhancing sample efficiency and policy learning in robotic manipulation. Chen et al. \citep{10529514} integrates an LLM-generated rule-based controller with RL, leveraging the collected data to create an imitation learning loss that guides policy updates.

However, most existing methods focus on discrete action spaces, while continuous control approaches often require human intervention, which is impractical in parallelized or high-acceleration environments and infeasible for complex tasks like bipedal locomotion. Besides, current LLM-aided RL algorithms often lack rigorous theoretical analysis. GuidedSAC bridges this gap by theoretically and empirically validating RL with suboptimal action-level guidance, offering a scalable and efficient framework for continuous control advice-taking agents.

\section{Conclusion}


This paper introduces GuidedSAC, which employs an LLM-based supervisor for action-level guidance to facilitate targeted exploration in complex continuous control tasks. We demonstrate, theoretically and empirically, that integrating this LLM-driven guidance into the SAC algorithm preserves convergence properties while enhancing convergence speed and overall performance. Our theoretical analysis identifies conditions under which external guidance is most beneficial to RL. Experimental results on discrete toy text tasks and continuous control benchmarks show that GuidedSAC enables efficient learning and promotes more reasonable policies. GuidedSAC presents a promising direction for developing exploration strategies in continuous control that intelligently explore high-value states or trajectories.


\section*{Impact Statement}
This paper presents work whose goal is to advance the field of Machine Learning. There are many potential societal consequences of our work, none which we feel must be specifically highlighted here.



\bibliography{guidedsac}
\bibliographystyle{plainnat}

\onecolumn
\appendix

\section{Proof of Convergence}
\label{appendix:proof}
\begin{lem}[Guided Policy Evaluation] \label{lem:guided_policy_evaluation}
    Consider the guided Bellman backup operator $\mathcal{T}^{\widetilde{\pi}}$ and a mapping $Q:\mathcal{S}\times\mathcal{A}\rightarrow\mathbb{R}$, 
    and define $Q^{k+1}=\mathcal{T}^{\widetilde{\pi}} Q^k$. Then the sequence $Q^k$ will converge to the Q-value of $\widetilde{\pi}$ as $k\rightarrow\infty$.
\end{lem}
\begin{proof}
    We expand the guided Bellman backup operator $\mathcal{T}^{\widetilde{\pi}}$ as
    \begin{equation} \label{eq:policy_evaluation}
        Q({s}, {a}) \leftarrow r\left({s}, {a}\right) + \gamma \mathbb{E}_{{s'} \sim p, {a'} \sim {\widetilde{\pi}},}\left[Q\left({s'}, {a'}\right)\right].
    \end{equation}
    Then, we can prove guided Bellman backup operator $\mathcal{T}^{\widetilde{\pi}}$ is a contraction mapping.
    \begin{equation}
        \begin{aligned}
        &\Vert \mathcal{T}^{\widetilde{\pi}} Q_{k} - \mathcal{T}^{\widetilde{\pi}} Q_{k+1} \Vert_\infty = \mathop{\max}_{s,a} \vert \mathcal{T}^{\widetilde{\pi}} Q_{k}(s,a) - \mathcal{T}^{\widetilde{\pi}} Q_{k+1}(s,a) \vert \\
        & \leq \gamma \mathop{\max}_{s,a} \lvert \mathbb{E}_{s',a'} \left[Q_{k}(s',a') - Q_{k+1}(s',a')\right] \rvert \\
        & \leq \gamma \mathop{\max}_{s,a} \vert Q_{k}(s',a') - Q_{k+1}(s',a') \vert \\
        & = \gamma \Vert Q_{k} - Q_{k+1} \Vert_\infty
        \end{aligned}
    \end{equation}
\end{proof}

\begin{lem}[Guided Policy Improvement]
    Let $\widetilde{\pi}_{\mathrm{old}} \in \Pi$ and update policy following Equation~(\ref{eq:policy_update_app}) to get $\pi_{\mathrm{new}}$. 
    Then $V^{\pi_{\mathrm{new}}}\left(s_t\right) \geq V^{\widetilde{\pi}_{\mathrm{old}}}\left(s_t\right)$ for all $s_t\in \mathcal{S}$.
\end{lem}

\begin{equation} \label{eq:policy_update}
    \begin{aligned}
\pi_{\mathrm{new}}\left(\cdot\mid s_t\right) & =\arg \min _{\pi^{\prime} \in \Pi} D_{\mathrm{KL}}\left(\pi^{\prime}\left(\cdot \mid s_t\right) \Bigg\Vert 
\frac{\exp Q^{\widetilde{\pi}_{\mathrm{old}}}\left(s_t, \cdot\right)}{\log Z^{\widetilde{\pi}_{\mathrm{old}}}\left(s_t\right)}\right) \\
& =\arg \min _{\pi^{\prime} \in \Pi} J_{\widetilde{\pi}_{\mathrm{old}}}\left(\pi^{\prime}\left(\cdot \mid s_t\right)\right),\ \forall s_t \in S .
\end{aligned}
\end{equation}

\begin{proof}
Let $\widetilde{\pi}_{\mathrm{old}} \in \Pi$ and let $Q^{\widetilde{\pi}_{\mathrm{old}}}$ and $V^{\widetilde{\pi}_{\mathrm{old}}}$ be the corresponding soft state-action value and soft state value, be defined as
    \begin{equation}
        \begin{aligned}
            &\mathbb{E}_{a_t \sim \pi_{\mathrm{new}}}\left[\log \pi_{\mathrm{new}}\left(a_t \mid s_t\right)-Q^{\widetilde{\pi}_{\mathrm{old}}}\left(s_t, a_t\right)\right] \leq \mathbb{E}_{a_t \sim \widetilde{\pi}_{\mathrm{old}}}\left[\log \widetilde{\pi}_{\mathrm{old}}\left(a_t \mid s_t\right)-Q^{\widetilde{\pi}_{\mathrm{old}}}\left(s_t, a_t\right)\right].
        \end{aligned}
    \end{equation}
    $J_{\widetilde{\pi}_{\mathrm{old}}}\left(\pi_{\mathrm{new}}\left(\cdot \mid s_t\right)\right) 
    \leq J_{\widetilde{\pi}_{\mathrm{old}}}\left(\widetilde{\pi}_{\mathrm{old}}\left(\cdot \mid s_t\right)\right)$ is guaranteed, since we can always choose $\pi_{\mathrm{new}}=\widetilde{\pi}_{\mathrm{old}} \in \Pi$.
    Since partition function $Z^{\widetilde{\pi}_{\mathrm{old}}}$ depends only on the state, the inequality reduces to
    \begin{equation} \label{eq:inequality}
        \mathbb{E}_{a_t \sim \pi_{\mathrm{new}}}\left[Q^{\widetilde{\pi}_{\mathrm{old}}}\left(s_t, a_t\right)-\log \pi_{\mathrm{new}}\left(a_t \mid s_t\right)\right] \geq V^{\widetilde{\pi}_{\mathrm{old}}}\left(s_t\right) .
    \end{equation}

Next, consider the Bellman equation for $Q^{\widetilde{\pi}_{\mathrm{old}}}$:
    \begin{equation}\label{eq:Bellman improvement}
        \begin{aligned}
            & Q^{\widetilde{\pi}_{\mathrm{old}}}\left(s_t, a_t\right) =r\left(s_t, a_t\right)+\gamma \mathbb{E}_{s_{t+1} \sim p}\left[V^{\widetilde{\pi}_{\mathrm{old}}}\left(s_{t+1}\right)\right] \\
            & \leq r\left(s_t, a_t\right)+\gamma \mathbb{E}_{s_{t+1} \sim p}\Big[\mathbb{E}_{a_{t+1} \sim \pi_{\mathrm{new}}}\Big[Q^{\widetilde{\pi}_{\mathrm{old}}}\left(s_{t+1}, a_{t+1}\right) -\log \pi_{\mathrm{new}}\left(a_{t+1} \mid s_{t+1}\right)\Big]\Big] \\
            & \vdots \\
            & \leq Q^{\pi_{\mathrm{new}}}\left(s_t, a_t\right),\ \forall \left(s_t, a_t\right) \in \mathcal{S} \times \mathcal{A}.
        \end{aligned}
    \end{equation}
    
Given $Q^{\pi_{\mathrm{new}}}\left(s_t, a_t\right) \geq Q^{\widetilde{\pi}_{\mathrm{old}}}\left(s_t, a_t\right)$ for all $\left(s_t, a_t\right) \in \mathcal{S} \times \mathcal{A}$, by definition, we have $V^{\pi_{\mathrm{new}}}\left(s_t\right) \geq V^{\widetilde{\pi}_{\mathrm{old}}}\left(s_t\right)$ for all $s_t\in \mathcal{S}$.
\end{proof}

\begin{thm}[Convergence of GuidedSAC]
\label{theorem:conv}
    By repeatedly applying the guided policy evaluation in Equation~(\ref{eq:policy_evaluation}) and guided policy improvement in Equation~(\ref{eq:policy_update_app}), 
    the policy network $\pi_{\mathrm{new}}$ converges to the optimal policy $\pi^*$ if $V^{\widetilde{\pi}}\left(s\right) \geq V^{\pi}\left(s\right)$ for all $s \in S$.
    \begin{proof}
        Applying $\mathcal{T}^{\widetilde{\pi}}$ repeatedly will induce $Q^k \rightarrow Q^{\widetilde{\pi}}$, also $V^k \rightarrow V^{\widetilde{\pi}}$. Given $V^{\widetilde{\pi}}\left(s\right) \geq V^{\pi}\left(s\right)$ for all $s \in S$, updating $\pi$ according to Equation~(\ref{eq:policy_update_app}) will result in $V^{\pi_{\mathrm{new}}}\left(s\right) \geq V^{\widetilde{\pi}_{\mathrm{old}}}\left(s\right) \geq V^{\pi_{\mathrm{old}}}\left(s\right)$, which leads to the monotonically increasing of $V^{\pi}\left(s\right)$. As $V^{\pi}\left(s\right)$ is bounded, $\pi$ will eventually converge to $\pi^*$.
    \end{proof}
\end{thm}

\section{Theoretical Analysis of Action-Level Intervention in Policy Gradient Methods}
\label{append:pg_intervene}

\noindent In this section, we analyze how an \emph{action-level intervention} mechanism changes the policy-gradient signal.
Consider a standard policy-gradient objective
\begin{equation}
    \label{eq:pg_objective_std}
    L^{\text{PG}}(\phi)\ \propto\ -\mathbb{E}_{\tau \sim \pi_\phi}\left[\sum_{t} A_t \log \pi_\phi(a_t \mid s_t)\right],
\end{equation}
where \(A_t\) denotes an advantage estimator.

\paragraph{Action-level intervention policy.}
Let $g(s) \in [0, 1]$ be a state-dependent intervention gate and $\pi_{\text{interv}}(\cdot \mid s)$ be a fixed intervention policy. We define the \emph{intervened} behavior policy as the mixture:
\begin{equation}
    \label{eq:intervened_policy_mixture_pg}
    \widetilde{\pi}_\phi(a \mid s) = g(s) \pi_{\text{interv}}(a \mid s) + (1 - g(s)) \pi_\phi(a \mid s).
\end{equation}
When $g(s) \in \{0, 1\}$, Equation~\eqref{eq:intervened_policy_mixture_pg} reduces to a hard switch: the agent either fully follows $\pi_{\text{interv}}$ or $\pi_\phi$ at any given state.

\paragraph{Policy gradient under intervention.}
When trajectories $\tau$ are sampled according to the intervened policy $\widetilde{\pi}_\phi$, the surrogate objective is defined as:
\begin{equation}
    \label{eq:pg_objective_intervened}
    L^{\text{PG}}(\phi) \propto -\mathbb{E}_{\tau \sim \widetilde{\pi}_\phi} \left[ \sum_{t} A_t \log \widetilde{\pi}_\phi(a_t \mid s_t) \right].
\end{equation}
Taking the gradient with respect to the parameters $\phi$ yields:
$$\begin{aligned}
    \nabla_{\phi} L^{\text{PG}}(\phi) &\propto -\mathbb{E}_{\tau \sim \widetilde{\pi}_\phi} \left[ \sum_{t} A_t \frac{\nabla_{\phi} \widetilde{\pi}_{\phi}(a_t \mid s_t)}{\widetilde{\pi}_{\phi}(a_t \mid s_t)} \right] \\
    &= -\mathbb{E}_{\tau \sim \widetilde{\pi}_\phi} \left[ \sum_{t} A_t \frac{g(s_t) \nabla_{\phi} \pi_{\text{interv}}(a_t \mid s_t) + (1 - g(s_t)) \nabla_{\phi} \pi_{\phi}(a_t \mid s_t)}{\widetilde{\pi}_{\phi}(a_t \mid s_t)} \right].
\end{aligned}$$
Noting that $\nabla_{\phi} \pi_{\text{interv}}(a_t \mid s_t) = 0$, the gradient simplifies to:
$$\begin{aligned}
    \nabla_{\phi} L^{\text{PG}}(\phi) &\propto -(1 - g(s_t)) \mathbb{E}_{\tau \sim \widetilde{\pi}_\phi} \left[ \sum_{t} A_t \nabla_{\phi} \log \pi_{\phi}(a_t \mid s_t) \right] \\
    &= -\mathbb{E}_{\tau \sim \widetilde{\pi}_\phi} \left[ \sum_{t} A_t \frac{(1 - g(s_t)) \pi_{\phi}(a_t \mid s_t)}{\widetilde{\pi}_{\phi}(a_t \mid s_t)} \nabla_{\phi} \log \pi_{\phi}(a_t \mid s_t) \right].
\end{aligned}$$
In the case of a hard intervention where $g(s) \in \{0, 1\}$, the gradient contribution vanishes whenever $g(s_t) = 1$. Conversely, when $g(s_t) = 0$, the expression reduces to the standard policy gradient. Under these conditions, the gradient simplifies to:
\begin{equation}
    \nabla_{\phi} L^{\text{PG}}(\phi) \propto -\mathbb{E}_{s_t \sim \rho^{\widetilde{\pi}_\phi}, \, g(s_t)=0, \, a_t \sim \pi_\phi} \left[ \sum_{t} A_t \nabla_{\phi} \log \pi_{\phi}(a_t \mid s_t) \right].
\end{equation}
Thus, action-level intervention effectively acts as a \emph{mask} on policy updates, only non-intervened steps provide a learning signal for $\phi$. Prior work on intervention-aided policy optimization shows that, the actor receives no learning signal after an intervention occurs. As a result, the learned policy can remain unreliable in states that frequently trigger interventions, harming robustness. A common remedy is to add an auxiliary imitation loss that matches the learner to the intervention actions on intervened steps \citep{wang2018intervention}.

\section{Implementation Details}

We use stable-baselines3\footnote{\url{https://github.com/DLR-RM/stable-baselines3}} to implement GuidedSAC. We use MLP networks with two hidden layers for all networks. Feature extractor is FlattenExtractor, which flattens the input features into a vector.

\subsection{Hyperparameters}

\begin{table}[tbhp]
    \centering
    \caption{Hyperparameter settings for GuidedSAC across different environments.}
    \label{tab:hyperparameters}
    \resizebox{0.7\textwidth}{!}{
    \begin{tabular}{lccccccc}
    \toprule
    \textbf{Environment} & Guided Window & End Step & Batch Size & $\gamma$ & $\tau$ & $\alpha$ & Start \\ \midrule
    Blackjack & 2000 & 5000 & 256 & 0.5 & 0.005 & auto & 1000 \\
    CliffWalking & 2000 & 10000 & 1024 & 0.99 & 0.005 & 0.01 & 0 \\
    FrozenLake & 100 & 5000 & 64 & 0.99 & 0.005 & 0.01 & 0 \\
    Taxi & 2000 & 5000 & 256 & 0.99 & 0.005 & 0.01 & 0 \\
    MountainCar & 1000 & 3000 & 256 & 0.99 & 0.005 & auto & 0 \\
    Humanoid & 1000 & 900,000 & 256 & 0.99 & 0.005 & auto & 100 \\ \bottomrule
    \end{tabular}
    }
\end{table}



Querying the Advisor LLM at every time step is expensive in practice. To solve this, we use a window mechanism where each intervention lasts for a fixed duration of steps. This guided window defines how long the agent follows the LLM guidance before the Advisor is consulted again to evaluate performance. This approach reduces API latency and ensures the agent collects experience under a consistent policy without violating our theoretical modeling.

Several hyperparameters govern the coordination between the LLM and the agent. The guided window allows for periodic performance checks while the end guidance timestep sets the final point where all external help is disabled. Standard learning parameters include the batch size for buffer sampling and the discount factor gamma ($\gamma$) to prioritize long-term rewards. Training stability is maintained by the soft update coefficient tau ($\tau$) and the entropy coefficient alpha ($\alpha$) which balances exploration and exploitation. Finally, the learning starts parameter defines the initial phase of random exploration used to fill the replay buffer before gradient updates begin.

The intrinsic reward coefficients are set to $10^{-4}$ for most experiments. For CliffWalking and Taxi, this coefficient is increased to $1$.

\begin{figure}[btp]
    \centering
    \includegraphics[width=1.0\linewidth]{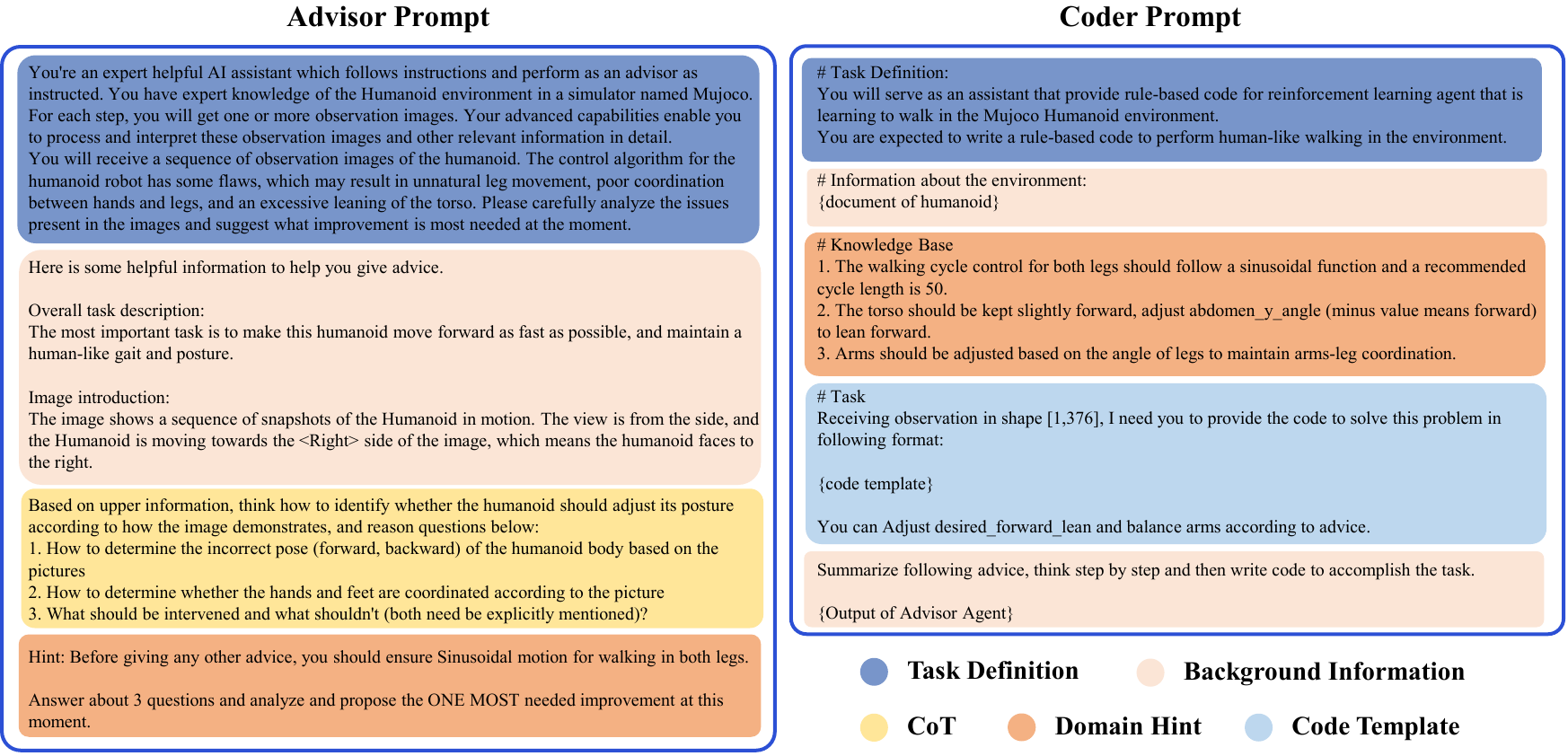}
    \caption{Prompts for the advisor and coder LLMs.}
    \label{fig:prompts}
\end{figure}

\subsection{LLM Configuration}
The proposed framework incorporates two specialized LLMs to address distinct functional requirements. We utilize the qwen3-vl-plus model as the Advisor LLM because of its advanced visual perception and spatial reasoning capabilities. For the Coder LLM, we employ the qwen3-max-preview model to leverage its superior performance in automated code synthesis and algorithmic generation. Both models are integrated into our experimental pipeline through remote API calls.

\subsection{Prompt Details} \label{sec:prompt_details}
Fig.~\ref{fig:prompts} shows the prompts for the advisor and coder LLMs in Humanoid task. The prompt for the advisor LLM is designed to evaluate whether intervention is necessary based on recent trajectory performance. The prompt for the coder LLM is designed to generate rule-based policies with sufficient context information when intervention is triggered.
\section{Pseudo-Code of GuidedSAC}

\begin{algorithm}[h]
\caption{Guided Soft Actor-Critic}
\label{alg:soft_actor_critic}
\begin{algorithmic}[1]
\STATE Initialize $\psi$, $\bar{\psi}$, $\theta$, $\phi$, $intervene=False$;
\FOR{each iteration}
    \IF{$intervene$}
        \FOR{each environment step}
        \STATE $\widetilde{a}_t \sim \widetilde{\pi}_\phi(a_t|s_t)$;
        \STATE $s_{t+1} \sim p(s_{t+1}|s_t, \widetilde{a}_t)$;
        \STATE $\widetilde{\mathcal{D}} \gets \widetilde{\mathcal{D}} \cup \{(s_t, \widetilde{a}_t, r(s_t, \widetilde{a}_t), s_{t+1})\}$;
        \ENDFOR
    \ELSE
        \FOR{each environment step}
        \STATE $a_t \sim \pi_\phi(a_t|s_t)$;
        \STATE $s_{t+1} \sim p(s_{t+1}|s_t, a_t)$;
        \STATE $\widetilde{\mathcal{D}} \gets \widetilde{\mathcal{D}} \cup \{(s_t, a_t, r(s_t, a_t), s_{t+1})\}$;
        \ENDFOR
    \ENDIF
    \STATE $intervene \gets I(s_{\leq T})$;
    \FOR{each gradient step}
        \STATE $\psi \gets \psi - \lambda_V \nabla_\psi J_V(\psi)$;
        \STATE $\theta \gets \theta - \lambda_Q \nabla_{\theta} J_Q(\theta)$;
        \STATE $\phi \gets \phi - \lambda_\pi \nabla_\phi J_\pi(\phi)$;
        \STATE $\bar{\psi} \gets \tau \psi + (1 - \tau) \bar{\psi}$;
    \ENDFOR
\ENDFOR
\end{algorithmic}
\end{algorithm}

\section{Limitation \& Future Work}
The performance gains from an LLM-based supervisor fundamentally rely on the LLM's ability to propose effective solutions that can be distilled into a rule-based policy. Such a rule-based policy represents a trade-off between cost and effectiveness: it avoids querying the LLM at every environment step, but its representational capacity is limited, and it can struggle to process rich observations (e.g., raw images) or to coordinate high-dimensional behaviors. For example, in vision-centric environments like Minecraft, the lack of sophisticated visual processing within the rule-based controller makes it difficult to extract effective policies from image-based states. Furthermore, in high-dimensional continuous control tasks such as the Unitree G1\footnote{\url{https://github.com/google-deepmind/mujoco_menagerie/tree/main/unitree_g1}} (29 dimensions), even if the LLM can understand the environment and provide reasonable advices in a subset of joints, the resulting rule-based intervention may have only marginal impact on overall performance, making exploration of high-value trajectories less likely. In the future, as LLM inference costs and latency continue to decline, it may become practical to deploy powerful LLMs for direct real-time interventions, thereby circumventing these limitations. Importantly, the theoretical framework developed in this work remains fully applicable to such settings.

\end{document}